\documentclass[11pt]{article}
\usepackage[T1]{fontenc}
\usepackage{lmodern}
\usepackage[margin=1in]{geometry}
\usepackage{amsmath,amssymb,amsthm}
\usepackage{booktabs,array}
\usepackage{microtype}
\usepackage{tikz}
\usepackage{hyperref}
\hypersetup{colorlinks=true,linkcolor=blue!45!black,citecolor=blue!45!black,urlcolor=blue!45!black,
 pdftitle={An Exact 93-Move Solution to a Nonlocal Crystal-Exchange Puzzle},
 pdfauthor={Zilai Wang}}
\newtheorem{theorem}{Theorem}[section]
\newtheorem{lemma}[theorem]{Lemma}
\newtheorem{proposition}[theorem]{Proposition}
\newtheorem{corollary}[theorem]{Corollary}
\theoremstyle{definition}
\newtheorem{definition}[theorem]{Definition}
\theoremstyle{remark}

\newcommand{\norm}[1]{\left\lVert #1\right\rVert_1}
\newcommand{\dist}{\operatorname{dist}}
\title{An Exact 93-Move Solution to a\\Nonlocal Crystal-Exchange Puzzle}
\author{Zilai Wang\\
 \small\href{mailto:wzlfly87@gmail.com}{\texttt{wzlfly87@gmail.com}}}
\date{October 21, 2026}
\begin{document}
\maketitle

\begin{abstract}
We determine the minimum number of directional commands needed to solve a
level of \emph{Order of the Sinking Star}, modeled as a crystal-exchange
puzzle on an \(11\times11\) board. A command swaps
the player with the nearest crystal in the selected row or column, regardless
of distance; only an empty ray permits an ordinary unit walk. Two gates admit
the player to the exit exactly while all sixteen crystals occupy their target
cells. We give an explicit 93-command solution and exclude every shorter
solution by an exact, cost-bounded breadth-first search. The search represents
a position by its seventeen occupied cells and one active cell. Its admissible
heuristic partitions command costs between a compressed core abstraction and
residual motion, strengthened by a necessary activation detour and a corner
visit condition. The completed computation retained 605,349,215 occupied sets
and 1,863,925,136 states. We prove the reductions and pruning conditions,
describe the exact storage scheme, and provide source code, execution logs,
and a reproduction driver. The result is a computer-assisted proof for the
specified instance, with an explicit implementation trust boundary.
\end{abstract}

\section{Introduction}
\label{sec:intro}
We study a crystal-exchange level from the puzzle game
\emph{Order of the Sinking Star}.
Nonlocal exchanges make familiar geometric intuitions unreliable in
shortest-path puzzles. A single command can move the player across an entire
row, while placing every crystal on a target need not finish the puzzle:
the player must still reach an exit. The position of the player at the moment
the gates open is therefore part of the optimization problem.

The level was supplied as a diagram together with
its movement rules. Section~\ref{sec:model} gives a self-contained
mathematical specification, so the result does not depend on access to a
particular game implementation. Our main result is an exact minimum of
93 directional commands. We give a witness and a computational exclusion
of all routes of length at most 92.

The computational method uses established principles of admissible search.
Pattern databases store exact distances in abstractions of a state
space~\cite{culberson1998}. Additive pattern databases and action-cost
partitioning permit several lower bounds to be summed when their allocated
costs do not exceed the cost of a real action~\cite{felner2004,katz2008}.
Here the partition is fixed and binary: a command contributes to exactly
one of two cost components. We specialize these principles to nonlocal
exchange by retaining occupied points, rather than crystal identities,
and deriving a residual-motion bound from the need to activate an
off-target occupied point.

Our contribution is the exact solution of this instance and a documented
search formulation that makes the result reproducible. We do not claim a
new general-purpose search algorithm, a complexity classification of the
puzzle family, or a proof checked by a formal proof assistant.

\section{The instance and its exact state space}
\label{sec:model}
Rows and columns are numbered from 1 to 11, with rows increasing downward.
Let
\[
 V=\{1,\ldots,11\}^2,\qquad K=\{4,\ldots,8\}^2,
\]
and let the sixteen target cells form the perimeter
\[
 T=\{(r,c)\in K:r\in\{4,8\}\ \text{or}\ c\in\{4,8\}\}.
\]
The exit is \(E=(11,11)\), the gates are
\(G=\{(10,11),(11,10)\}\), and the initial positions are
\begin{equation}
 p_0=(10,10),\qquad
 C_0=(T\setminus\{(4,4)\})\cup\{(1,1)\}.
 \label{eq:initial}
\end{equation}
Crystals are indistinguishable. A state consists of a player cell \(p\)
and a sixteen-element set \(C\), with \(p\notin C\).

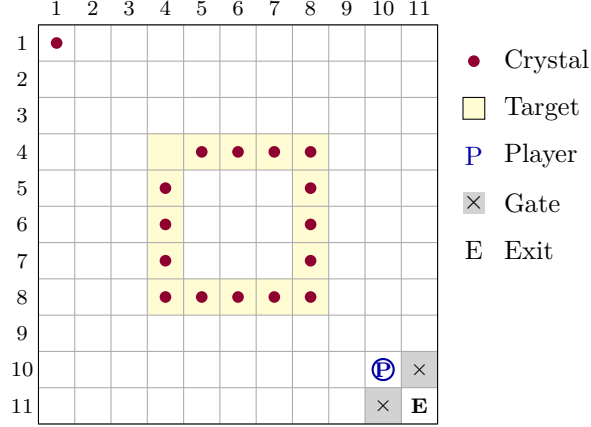
\begin{figure}[tb]
\centering
\begin{tikzpicture}[x=.48cm,y=-.48cm,font=\small]
 \foreach \r in {4,...,8} {
  \foreach \c in {4,...,8} {
   \ifnum\r=4 \fill[yellow!22] (\c-1,\r-1) rectangle (\c,\r);
   \else\ifnum\r=8 \fill[yellow!22] (\c-1,\r-1) rectangle (\c,\r);
   \else\ifnum\c=4 \fill[yellow!22] (\c-1,\r-1) rectangle (\c,\r);
   \else\ifnum\c=8 \fill[yellow!22] (\c-1,\r-1) rectangle (\c,\r);
   \fi\fi\fi\fi
  }
 }
 \fill[gray!35] (10,9) rectangle (11,10);
 \fill[gray!35] (9,10) rectangle (10,11);
 \foreach \k in {0,...,11} {
  \draw[gray!65,thin] (\k,0)--(\k,11);
  \draw[gray!65,thin] (0,\k)--(11,\k);
 }
 \draw[black] (0,0) rectangle (11,11);
 \foreach \k in {1,...,11} {
  \node[font=\scriptsize] at (\k-.5,-.45) {\k};
  \node[font=\scriptsize] at (-.45,\k-.5) {\k};
 }
 \foreach \r/\c in {1/1,4/5,4/6,4/7,4/8,5/4,5/8,6/4,6/8,7/4,7/8,8/4,8/5,8/6,8/7,8/8}
  \fill[purple!75!black] (\c-.5,\r-.5) circle (.16);
 \draw[blue!65!black,thick] (9.5,9.5) circle (.29);
 \node[blue!65!black,font=\scriptsize\bfseries] at (9.5,9.5) {P};
 \node[font=\scriptsize] at (10.5,9.5) {\(\times\)};
 \node[font=\scriptsize] at (9.5,10.5) {\(\times\)};
 \node[font=\scriptsize\bfseries] at (10.5,10.5) {E};
 \begin{scope}[shift={(12,1)}]
  \fill[purple!75!black] (0,0) circle (.16);
  \node[anchor=west] at (.55,0) {Crystal};
  \fill[yellow!22] (-.3,1) rectangle (.3,1.6);
  \draw (-.3,1) rectangle (.3,1.6);
  \node[anchor=west] at (.55,1.3) {Target};
  \node[blue!65!black] at (0,2.6) {P};
  \node[anchor=west] at (.55,2.6) {Player};
  \fill[gray!35] (-.3,3.6) rectangle (.3,4.2);
  \node at (0,3.9) {\(\times\)};
  \node[anchor=west] at (.55,3.9) {Gate};
  \node at (0,5.2) {E};
  \node[anchor=west] at (.55,5.2) {Exit};
 \end{scope}
\end{tikzpicture}
\caption{The initial instance, redrawn from its mathematical specification.
The target at \((4,4)\) is empty; the displaced crystal is at \((1,1)\).
Both gates are initially closed.}
\label{fig:board}
\end{figure}

\paragraph{Commands.}
The four directions are denoted \(U,D,L,R\). The gates are open if and
only if \(C=T\); their status is not latched. In a selected
direction, scan the ray from \(p\), stopping at the board boundary or a
closed gate. If the ray contains a crystal, let \(q\) be the nearest one
and replace the state by
\[
 \bigl(q,\ (C\setminus\{q\})\cup\{p\}\bigr).
\]
If the ray contains no crystal, move the player one cell in that direction,
provided that cell is in \(V\) and is not a closed gate. Each command costs
one, including a long-distance exchange. Invalid commands can be omitted
from a shortest route; if interpreted as idle commands, deleting them
does not change any subsequent state. The objective is the player's
\emph{first entry} into \(E\).

\begin{lemma}[Goal characterization]
\label{lem:goal}
On first entry into \(E\), the crystal set is \(T\).
\end{lemma}
\begin{proof}
Initially \(E\) has no crystal. The only operation that could place a
crystal there is an exchange from a player already at \(E\). Thus a first
entry cannot be an exchange; it is a walk from one of the two gates.
Entering that gate requires \(C=T\). From either gate, its row and column
contain no target cells. While the player remains on a gate or walks
between it and \(E\), no exchange can alter \(T\). Hence \(C=T\) on entry.
\end{proof}

\begin{definition}[Occupied-set representation]
Write \(U=C\cup\{p\}\) and \(s=(U,p)\).
Then \(|U|=17\), \(p\in U\), and \(C=U\setminus\{p\}\).
\end{definition}
An exchange changes only the active cell \(p\). A walk from \(p\) to \(q\)
replaces \(U\) by \((U\setminus\{p\})\cup\{q\}\). This is an exact
representation, not an abstraction. By Lemma~\ref{lem:goal}, the required
terminal state is
\begin{equation}
 s_*=(T\cup\{E\},E).
 \label{eq:goal}
\end{equation}

\begin{theorem}[Main result; computer-assisted]
\label{thm:main}
For the instance and rules specified above, the minimum number of commands
from \((C_0\cup\{p_0\},p_0)\) to first entry into \(E\) is exactly \(93\).
\end{theorem}
The upper bound is explicit in Section~\ref{sec:result}. The lower bound
combines the mathematical reductions below with the exhaustive
computation reported there.

\section{Partitioning command costs}
\label{sec:partition}
For a one-dimensional coordinate \(k\in\{1,\ldots,11\}\), define
\begin{equation}
 \pi(k)=\min\{6,\max\{0,k-3\}\},\qquad
 \eta(k)=
 \begin{cases}
  k-3,&k<3,\\
  0,&3\le k\le9,\\
  k-9,&k>9.
 \end{cases}
 \label{eq:coordinates}
\end{equation}
Apply both maps coordinatewise to cells. The map \(\pi\) compresses the
board to \(Q=\{0,\ldots,6\}^2\); \(\eta\) records the signed residual
displacement outside the interval \([3,9]\). In one dimension,
\(k=\pi(k)+\eta(k)+3\).

\begin{table}[tb]
\centering
\small
\setlength{\tabcolsep}{7pt}
\begin{tabular}{c*{11}{r}}
\toprule
\(k\)&1&2&3&4&5&6&7&8&9&10&11\\
\midrule
\(\pi(k)\)&0&0&0&1&2&3&4&5&6&6&6\\
\(\eta(k)\)&\(-2\)&\(-1\)&0&0&0&0&0&0&0&1&2\\
\bottomrule
\end{tabular}
\caption{Every adjacent pair changes exactly one of \(\pi\) and \(\eta\)
by one. The identity \(k=\pi(k)+\eta(k)+3\) alone would not establish
this property; the adjacent values do.}
\label{tab:projection}
\end{table}

\begin{lemma}[Disjoint cost partition]
\label{lem:partition}
A unit walk changes exactly one coordinate of either \(\pi(p)\) or
\(\eta(p)\), by one. Charge every exchange and every \(\pi\)-changing walk
to cost \(\alpha\), and every \(\eta\)-changing walk to cost \(\beta\).
For every route of length \(L\), \(L=\alpha+\beta\).
\end{lemma}
\begin{proof}
In one dimension the edges \(3\leftrightarrow4,\ldots,8\leftrightarrow9\)
change only \(\pi\); the remaining four edges change only \(\eta\).
A grid walk changes only one original coordinate. The asserted partition
therefore assigns each walk to exactly one class, and exchanges are
assigned to \(\alpha\) by definition.
\end{proof}

\paragraph{Stationary labels under exchange.}
Temporarily label the seventeen occupied points. A label stays at its
cell during an exchange; only which label is active changes. During a walk,
the active label moves with the player. These are bookkeeping labels of
occupied points, not persistent identities of crystals. Consequently,
only walks move labels, and only \(\beta\)-walks change their residual
coordinates. At completion, sixteen labels have residual coordinate
\((0,0)\), and one has residual coordinate \((2,2)\). This observation
permits a displacement lower bound on \(\beta\), even though the
player itself can change residual coordinates in a single exchange.

\section{A compressed-core lower bound}
\label{sec:core}
Let \(Q^\circ=\{1,\ldots,5\}^2\) be the interior of \(Q\).
The abstraction of \(s=(U,p)\) is
\[
 A(s)=(M,a,e),\qquad
 M=\pi(U\cap K),\quad a=\pi(p),\quad e=\pi(E)=(6,6).
\]
The endpoint corner \(e\) is stored explicitly because the pattern table
uses all square symmetries. The integer \(n_{\mathrm{out}}=17-|M|\)
counts occupied points outside \(K\), including the player
when appropriate. Their individual positions are forgotten. Valid states
satisfy \(|M|\le17\), \(a\in M\) when \(a\in Q^\circ\), and \(n_{\mathrm{out}}\ge1\)
when \(a\in\partial Q\).

\subsection{Abstract transitions}
The following directed transitions have unit cost. Gates are ignored.
Rays exclude their starting cell.
\begin{enumerate}
 \item \emph{Exchanges from the interior.}
 If \(a\in Q^\circ\), an exchange may reach the first point of \(M\)
 on a chosen ray. If there is none, it may reach the boundary cell on
 that ray, provided \(n_{\mathrm{out}}\ge1\).
 \item \emph{Exchanges from the boundary.}
 If \(a\in\partial Q\), an exchange may reach the first point of \(M\)
 on a ray. It may also reach any other boundary cell on that ray before
 a point of \(M\), provided \(n_{\mathrm{out}}\ge2\).
 Exchanges preserve \(M\).
 \item \emph{Walks.}
 The active cell may move to an adjacent \(b\in Q\) if the ray from
 \(a\) in that direction contains no point of \(M\). Remove \(a\) from
 \(M\) if \(a\in Q^\circ\), and add \(b\) if \(b\in Q^\circ\).
 Only valid resulting states are retained.
\end{enumerate}
These rules permit configurations of exterior points that need not be
simultaneously realizable on the original board. That extra freedom is
intentional. Exchanges whose endpoints have the same projection need no
abstract transition.

\begin{lemma}[Simulation]
\label{lem:simulation}
Every real \(\alpha\)-command maps to one abstract edge or no change.
Every real \(\beta\)-command maps to no change.
\end{lemma}
\begin{proof}
A real exchange cannot pass a nearer interior crystal. If it reaches an
exterior point, the corresponding boundary destination is allowed by
the occupied-point count: an exterior-to-exterior exchange uses two
distinct exterior points. Projection preserves row or column alignment
and order along the noncollapsed part of a ray. Thus every exchange
with distinct projected endpoints is included.

A real walk is allowed only when its entire unblocked ray has no
crystal. The only cells hidden behind closed gates are gate cells and
exit-side boundary cells; none belongs to \(K\). Hence there is no
interior occupied point ahead of a legal walk. A \(\pi\)-changing walk
satisfies the abstract walk rule, with precisely the stated
update of \(M\). An \(\eta\)-changing walk stays outside \(K\), does not
change \(\pi(p)\), and leaves \(M\) fixed.
\end{proof}

Let \(T'=\pi(T)\), and define the base pattern distance
\[
 d(A)=\dist_{\mathcal A}\bigl(A,(T',e,e)\bigr),
\]
where \(\mathcal A\) is the directed abstract graph. Reverse
breadth-first search from the abstract goal computes \(d\).
Lemma~\ref{lem:simulation} gives \(d(A(s))\le\alpha\) for every
completion of \(s\).

\subsection{A necessary northwest visit}
\label{sec:landmark}
Let \(N=\{1,2,3\}^2\), whose projection is the northwest corner \(o=(0,0)\).
If \(U\cap N\ne\varnothing\), at least one currently occupied point in
\(N\) must eventually move: neither a target nor the exit lies there.
Before its first walk, that point must be active. Thus the projected
player must visit \(o\), with a visit at the current time permitted.

For an abstract state \(A=(M,a,e)\), let \(o(e)\) be the corner opposite
\(e\), and define
\begin{equation}
 d_N(A)=
 \min_{\substack{B=(M',o(e),e)\\ B\text{ valid}}}
 \bigl[\dist_{\mathcal A}(A,B)+d(B)\bigr].
 \label{eq:landmark}
\end{equation}
This value is computed by reverse Dijkstra search, seeding each such
\(B\) with its base distance \(d(B)\). Set
\begin{equation}
 H_\alpha(s)=
 \begin{cases}
 d_N(A(s)),&U\cap N\ne\varnothing,\\
 d(A(s)),&U\cap N=\varnothing.
 \end{cases}
 \label{eq:halpha}
\end{equation}

\begin{proposition}
\label{prop:halpha}
For every completion of \(s\), \(H_\alpha(s)\le\alpha\).
\end{proposition}
\begin{proof}
Without the northwest condition this is the simulation lemma. With that
condition, split the projected route at an occurrence of \(o\).
The abstract prefix and suffix have total length at most the real
\(\alpha\)-cost and are candidates in \eqref{eq:landmark}.
\end{proof}

\subsection{Symmetry and pattern-table representation}
The group \(D_4\) acts on the abstract board. We transform \(M,a,e\)
together; \(o(e)\) transforms with \(e\). Canonicalizing the mask alone
without retaining goal orientation would not justify the distance lookup.
After choosing a representative mask, the implementation minimizes the
pair \((a,e)\) over the mask's stabilizer.

There are 4,211,744 orbits of all \(2^{25}\) masks under \(D_4\).
Each mask representative has \(49\cdot4=196\) addressable player/corner
slots, so each byte-valued distance array contains 825,501,824 slots.
The arrays include unused and noncanonical addresses. Each generated
array has 584,039,298 entries different from the sentinel 255.
These are storage counts, not counts of all real positions.
For the bound-93 computation, sentinel 255 is sufficient to reject a
state: all abstract paths of cost at most 92 are represented by the
reverse searches. Values at or beyond the sentinel need not be
distinguished. The initial core lower bound used in the search is 65.

\section{Residual motion and activation detours}
\label{sec:residual}
Fix labels \(1,\ldots,17\) on the occupied points of the current state.
Write \(r_i=\eta(u_i)\), and let \(i_p\) be the active label.
For each possible exit-assigned label \(e\), put
\[
 z_i^{(e)}=
 \begin{cases}(2,2),&i=e,\\(0,0),&i\ne e,\end{cases}
 \qquad D_e=\sum_{i=1}^{17}\norm{r_i-z_i^{(e)}}.
\]
Every completion induces one such assignment. Since each \(\beta\)-walk
moves exactly one residual coordinate by one,
\begin{equation}
 D_0=\min_e D_e
 =\sum_i\norm{r_i}+
   \min_e\bigl(\norm{r_e-(2,2)}-\norm{r_e}\bigr)
 \label{eq:baseline}
\end{equation}
is a lower bound on its residual-motion cost.

\subsection{One necessary activation}
Consider an inactive label \(v\) with \(r_v\ne(0,0)\) and \(u_v\ne E\).
It is currently on neither a target nor the exit and must eventually
move. Before its first walk, some distinct active label \(j\) must exchange
with it. Until that first activation, \(u_v\) remains fixed. Just before
the activating exchange, \(u_j\) and \(u_v\) share a row or column.
In particular, one residual coordinate of \(j\) equals the corresponding
coordinate of \(r_v\).

For a possible terminal residual coordinate \(z\), define the excess
length forced by this necessary alignment:
\begin{equation}
 \delta_j(z;v)=
 \min_{\ell\in\{1,2\}}
 \left(
 |r_{j,\ell}-r_{v,\ell}|+
 |r_{v,\ell}-z_\ell|-
 |r_{j,\ell}-z_\ell|
 \right).
 \label{eq:detour}
\end{equation}
It is nonnegative by the triangle inequality. Alignment in residual
coordinates is only a necessary condition for real alignment; weakening
the condition can only decrease this bound.
Allowing every possible exit assignment and approaching label gives
\begin{equation}
 D_v=\min_{\substack{e\in\{1,\ldots,17\}\\j\ne v}}
       \left[D_e+\delta_j\bigl(z_j^{(e)};v\bigr)\right].
 \label{eq:activation}
\end{equation}
In particular, the approaching label is allowed to be the exit-assigned
label. Excluding that possibility would not be justified.
Define
\[
 \mathcal V(s)=\{v:v\ne i_p,\ r_v\ne(0,0),\ u_v\ne E\}
\]
and
\begin{equation}
 H_\beta(s)=\max\bigl(\{D_0\}\cup\{D_v:v\in\mathcal V(s)\}\bigr).
 \label{eq:hbeta}
\end{equation}

\begin{proposition}[Residual lower bound]
\label{prop:hbeta}
For every completion of \(s\), \(H_\beta(s)\le\beta\).
\end{proposition}
\begin{proof}
Fix the actual exit assignment \(e\). For any \(v\in\mathcal V(s)\),
let \(j\ne v\) be the label that first activates it. The residual path
of \(j\), from \(r_j\) to \(z_j^{(e)}\), passes through a point with at
least one coordinate equal to that of \(r_v\). Its length is therefore
at least its direct Manhattan distance plus
\(\delta_j(z_j^{(e)};v)\). Every other label requires at least its direct
distance. Summing these disjoint label motions gives the expression
inside \eqref{eq:activation} for this particular \(e,j\), and hence
\(\beta\ge D_v\). The same reasoning without a detour gives
\(\beta\ge D_0\). Taking the maximum proves the claim.
\end{proof}

The maximum in \eqref{eq:hbeta} is essential: one motion may help
activate several points, so their detour bounds cannot in general be
added. The implementation computes the strongest two obligations and
the cell identifying the strongest one. This permits the active label's
obligation to be omitted for any choice of \(p\) in the same \(U\).

\begin{lemma}[Safe reuse along walks]
\label{lem:lipschitz}
For a legal walk \(s\to s'\), the residual bound is unchanged if the
walk has \(\alpha\)-cost. If the walk has \(\beta\)-cost, then
\(\lvert H_\beta(s')-H_\beta(s)\rvert\le1\).
\end{lemma}
\begin{proof}
The active label remains the same and all inactive labels remain fixed.
Thus the candidate set \(\mathcal V\) and the activation coordinates
are unchanged. For fixed \(e,j,v\) and choice of alignment coordinate,
the expression \(D_e+\delta_j(z_j^{(e)};v)\) is a sum of path lengths.
If the active label is \(j\), its contribution is a path through a
fixed coordinate; otherwise its contribution is a direct distance.
Changing its residual start coordinate by one changes either length
by at most one. If its residual coordinate does not change, neither
does the expression. Minima and maxima preserve these properties.
The same argument applies to \(D_0\).
\end{proof}

\begin{corollary}[Admissible combined heuristic]
\label{cor:heuristic}
The function \(h(s)=H_\alpha(s)+H_\beta(s)\) is a lower bound on the
length of every completion of \(s\).
\end{corollary}
\begin{proof}
Apply Propositions~\ref{prop:halpha} and \ref{prop:hbeta} to the
\emph{same} completion and use \(L=\alpha+\beta\).
\end{proof}
The search uses an initial value \(h(s_0)=65+10=75\).
A separate analytic argument proves that every complete route from
the initial state has \(\beta\ge12\)
(Appendix~\ref{app:elementary}); that stronger initial-state statement
was not used as a statewise residual heuristic in the reported search.

\section{Exact bounded search}
\label{sec:search}
Fix a strict bound \(B=93\). Breadth-first search processes states in
nondecreasing prefix length \(g\), enumerating all four legal commands.
A successor at depth \(g+1\) is discarded when
\begin{equation}
 (g+1)+h(s')\ge B.
 \label{eq:prune}
\end{equation}
Previously visited states need not be inserted again. The mathematical
algorithm is shown below; implementation-specific caches do not change
its accepted state set.

\begin{center}
\begin{minipage}{.94\linewidth}
\small
\hrule
\vspace{.5em}
\textbf{Bounded breadth-first exclusion}
\begin{verbatim}
frontier := {initial state}; seen := {initial state}; g := 0
while frontier is not empty:
    next := empty
    for state in frontier:
        for direction in {U, D, L, R}:
            if the command is illegal: continue
            child := exact successor(state, direction)
            child := symmetry_representative(child)
            if g + 1 + h(child) >= B: continue
            if child is in seen: continue
            insert child into seen and next
            if child is the goal: return a shorter solution
    frontier := next; g := g + 1
return "no solution of length less than B"
\end{verbatim}
\hrule
\end{minipage}
\end{center}

\begin{theorem}[Soundness of exclusion]
\label{thm:exclusion}
If the bounded search exhausts its queue without finding the goal,
there is no solution of length less than \(B\).
\end{theorem}
\begin{proof}
Suppose a solution of length \(L<B\) exists. At any state on that route,
its prefix length plus \(h\) is at most \(L\), by
Corollary~\ref{cor:heuristic}; thus \eqref{eq:prune} cannot delete the
prefix. A duplicate state has already been reached at an equal or
smaller depth, and the same completion is available from that earlier
arrival. Following these earlier arrivals, if necessary, gives a
retained route to the goal. This contradicts exhaustion without a goal.
\end{proof}

\subsection{Symmetry of the real instance}
Only reflection \(\tau(r,c)=(c,r)\) is used to merge real positions.
It preserves the initial state, targets, pair of gates, exit, and all
transition costs. Store the smaller of the row-major bitsets \(U\)
and \(\tau(U)\), reflecting \(p\) at the same time. If the two bitsets
are equal, the implementation keeps both active positions separately.
This may retain redundant symmetric states but does not merge
inequivalent states. Paths between representatives lift to real paths
because reflection is a transition-system automorphism and fixes
the endpoints. The \(D_4\) reduction of pattern tables is distinct
from this two-element symmetry of the real problem.

\subsection{Lossless storage}
\label{sec:encoding}
Each occupied set stores a 25-bit mask on \(K\). Let
\(k=17-|U\cap K|\), enumerate the other 96 cells in row-major order,
and let \(0\le x_1<\cdots<x_k<96\) be its exterior indices. Their
combinatorial rank is
\begin{equation}
 R=\sum_{j=1}^k \binom{x_j}{j}.
 \label{eq:rank}
\end{equation}
For fixed \(k\), this is an injective encoding with
\(0\le R<\binom{96}{k}\). Together with the core mask, it determines
\(U\) uniquely. A 64-bit word stores the core mask and low rank bits;
all additional bits are stored in a separate map and participate in
key equality. At most 87 bits suffice. No key bits are dropped.

For a given \(U\), a 17-bit mask records which active positions have
been visited. Seven further bits store the first depth at which this
occupied set was encountered; that depth is used for witness
reconstruction. The bound \(B=93\) fits within this field. The main
arrays require eleven payload bytes per occupied set: eight for the
low key and three for this metadata. The hash index, exceptional key
bits, frontiers, pattern tables, and allocation overhead are additional.
Hash collisions are resolved by full-key comparison, not probabilistic
fingerprinting. Frontier entries use five bytes to identify an occupied
set and one of its seventeen active positions.

\subsection{Optimizations and restart invariants}
For a walk, Lemma~\ref{lem:lipschitz} permits an inexpensive preliminary
test using
\[
 H_\alpha(s')+\max\{0,H_\beta(s)-b\},
 \quad
 b=\begin{cases}0,&\alpha\text{-walk},\\1,&\beta\text{-walk}.\end{cases}
\]
This value is at most the full successor heuristic. If it already
exceeds the available budget, no hash lookup is needed. Otherwise the
full bound is used. A bounded cache stores occupied sets and their
residual obligations; a hit requires an exact, immutable set identifier.
It cannot change state equality or the result of a heuristic evaluation.

Checkpoint files retain the full visited collection, the current
frontier, the generated next frontier, the current layer, the processed
position within that layer, and all exceptional key bits. Rebuilding
the hash index changes only how these exact keys are found.
When a capacity stop occurs partway through a state's four successors,
resumption repeats that state; visited bits suppress already generated
children. No unprocessed successor is skipped.
Only normal queue exhaustion constitutes an exclusion result.

\section{Construction and completed computation}
\label{sec:result}
Table~\ref{tab:route} gives a 93-command route. A separate Python model
replays the exact ray rule and conditional gate rule. It agrees with
all 93 rows of the released step-by-step coordinate record.
The gates first open after command 90, when \(p=(8,11)\).
The final three downward walks reach \(E\).

\begin{table}[tb]
\centering
\begin{tabular}{rll}
\toprule
Commands & Direction sequence & Player after the block\\
\midrule
1--10 & \texttt{URULULLURR} & \((6,8)\)\\
11--20 & \texttt{RDLLLULLRD} & \((7,6)\)\\
21--30 & \texttt{LURLURRUDL} & \((7,8)\)\\
31--40 & \texttt{LLLULURURR} & \((4,7)\)\\
41--50 & \texttt{URUULDRDDD} & \((5,2)\)\\
51--60 & \texttt{RRDLDRDRRR} & \((8,7)\)\\
61--70 & \texttt{URUULDDDDR} & \((5,8)\)\\
71--80 & \texttt{URULLDLDRD} & \((7,7)\)\\
81--90 & \texttt{DLLURRRDDR} & \((8,11)\)\\
91--93 & \texttt{DDD} & \((11,11)\)\\
\bottomrule
\end{tabular}
\caption{A shortest route. Directions are up, down, left, and right.
Each letter costs one command, even when it causes a nonlocal exchange.}
\label{tab:route}
\end{table}

\begin{samepage}
The route uses 37 exchanges, 44 \(\pi\)-changing walks, and
12 \(\eta\)-changing walks. Thus \((\alpha,\beta)=(81,12)\).
It attains the analytic residual-motion lower bound of 12 and the
three-walk terminal bound in Appendix~\ref{app:elementary}.
Neither fact alone proves the total optimum: routes with a different
allocation between \(\alpha\) and \(\beta\) must still be excluded.
\par\end{samepage}

\begin{table}[tb]
\centering
\begin{tabular}{lr}
\toprule
Quantity & Recorded value\\
\midrule
Strict exclusion bound \(B\) & 93\\
Initial heuristic \(H_\alpha+H_\beta\) & \(65+10=75\)\\
Retained occupied sets & 605,349,215\\
Retained states \((U,p)\) & 1,863,925,136\\
Largest frontier (depth 27) & 179,737,970\\
Last nonempty retained layer & 74\\
Final configured occupied-set capacity & 650,000,000\\
Verified witness length & 93\\
\bottomrule
\end{tabular}
\caption{Deterministic counts from the completed strict-bound-93 run.
State counts include the initial state and use the representation and
symmetry convention of Section~\ref{sec:search}.}
\label{tab:results}
\end{table}

The search completed with the following output:
\begin{center}
\begin{minipage}{.96\linewidth}
\small\ttfamily
EXHAUSTED\_NO\_ROUTE\_BELOW 93\\
OCCUPIED\_SETS 605349215 REACHED 1863925136
\end{minipage}
\end{center}
This is the program's normal exhaustion outcome, returned with exit
code 1. Its final frontier, at depth 74, contained one state and
generated no retained successor.
The final layer being less than 92 is expected: the heuristic can
exclude a prefix long before it reaches the candidate length limit.
Every prefix that could complete within 92 commands remains covered
by Theorem~\ref{thm:exclusion}.

The recorded run was continued across complete checkpoints. An initial
storage implementation reached depth 30; an equivalent smaller
representation continued the same search. A later stop at a configured
capacity of 600,000,000 occupied sets occurred within depth 40.
After increasing the capacity to 650,000,000, the exact saved state was
restored and the queue was exhausted. Counts at the checkpoint
boundaries match.

\begin{samepage}
The consolidated layer record satisfies
\[
 \sum_{g=0}^{74} F_g=1,863,925,136.
\]
Here \(F_g\) is the initial frontier size of layer \(g\).
The complete per-layer counts and all three log segments are included
with the ancillary material.
\par\end{samepage}

\begin{proof}[Proof of Theorem~\ref{thm:main}]
The route in Table~\ref{tab:route}, verified by exact replay, gives
\(L_{\min}\le93\). The completed exhaustive computation and
Theorem~\ref{thm:exclusion} give \(L_{\min}\ge93\).
Therefore \(L_{\min}=93\).
\end{proof}

\section{Reproducibility and scope of the evidence}
\label{sec:repro}
The ancillary directory contains the C++17 search and pattern-table
generators, the independent Python transition model, the witness,
the raw logs, and a reproduction driver. The main entry point is
\texttt{crystal\_verify\_optimum.py}. From that directory, run
\begin{center}
\begin{minipage}{.96\linewidth}
\small\ttfamily
python3 crystal\_verify\_optimum.py --check-route-only\\[3pt]
python3 crystal\_verify\_optimum.py --work-dir /path/to/work
\end{minipage}
\end{center}
The first command checks only the upper bound. The second compiles the
programs, regenerates the symmetry and distance tables, and starts
the full exclusion from the initial state. It accepts the lower-bound
result only when both the process return code and the complete
exhaustion message agree with the strict bound. A capacity stop,
interruption, or a shorter solution does not produce a verification
success. Python 3.9 or later, a C++17 compiler supporting
\texttt{unsigned \_\_int128}, and a POSIX system with memory mapping
are required.

\paragraph{Validation and provenance.}
The released record includes source and table SHA-256 values, a complete
route replay, and checkpoint-continuity checks. Development checks
covered 100,000 occupancy-codec round trips, including high-bit cases,
and fresh-versus-resumed complete smaller-bound searches. The selected
storage engines agreed on a strict-bound-89 regression with
15,162,550 occupied sets and 40,547,600 states. These regression checks
support implementation consistency; the bound-93 exhaustion is the
computation that establishes the final lower bound.

\paragraph{Resources.}
The recorded computation ran on a macOS machine with 16\,GiB of physical
memory and incurred substantial memory pressure. The eleven-byte main
payload at the final set count is 6,658,841,365 bytes, excluding the
hash index and other arrays. The two distance arrays alone total
1,651,003,648 bytes. These are calculable data sizes, not measured peak
resident memory. Because development, restarts, and paging were
interleaved, we do not report the elapsed development session as a
controlled solver-runtime benchmark.

\paragraph{Trust boundary.}
The lower bound is a computer-assisted exhaustion argument. Its
mathematical part proves that the transitions, equivalences, and
pruning rules preserve every shorter solution. Its computational
part relies on the supplied implementations, compiler, runtime, and
recorded execution. The file named
\texttt{crystal\_optimality\_certificate.json} is a manifest of the
result and its provenance; it is not a succinct proof object that an
independent formal kernel can check. Hashes establish artifact identity,
not algorithmic correctness. The upper-bound replay uses a separate
implementation, but we do not claim an independently implemented
full lower-bound verification.

\paragraph{Scope.}
The theorem concerns the specified board and conditional-gate,
nearest-crystal rules. It does not establish optimality for variants
with adjacent-only exchange or permanently latched gates. Nor does
the instance-specific computation establish comparative performance
against other optimal-search methods. An independent reproduction
or a small proof-checking kernel would strengthen the evidence.

\paragraph{Statement on AI use.}
OpenAI Codex was used to write and revise the search programs and to assist
with the computational investigation and preparation of this manuscript.
The author has reviewed and understood all content of this manuscript,
including the mathematical arguments and the search implementation.
The reported results are fully reproducible from the mathematical
specification, proofs, source code, data, and reproduction instructions
provided with this paper.

\appendix
\section{Two elementary bounds}
\label{app:elementary}
\begin{proposition}[Residual cost from the initial state]
Every solution of the specified instance has \(\beta\ge12\).
\end{proposition}
\begin{proof}
Let \(A\) be the initially active label and \(B\) the label at \((1,1)\).
Their residual positions are \((1,1)\) and \((-2,-2)\); the other fifteen
labels have residual position \((0,0)\).
Let \(J\ne B\) first activate \(B\). One coordinate of \(J\)'s residual
path must reach \(-2\).

If the exit-assigned label is not \(A\), the sum of all direct residual
distances is 10. Whether \(J\) starts at \((1,1)\) or \((0,0)\), and
whether it ends at \((0,0)\) or \((2,2)\), visiting a coordinate \(-2\)
adds at least four to its direct distance. Thus \(\beta\ge14\).

Suppose instead that \(A\) is exit-assigned. If \(J=A\), the residual
path of \(A\), from \((1,1)\) through a coordinate \(-2\) to \((2,2)\),
has length at least eight. Label \(B\) requires at least four, so
\(\beta\ge12\).
If \(J\ne A\), it is one of the fifteen initially central labels.
It starts and ends at \((0,0)\) but visits a coordinate \(-2\), costing
at least four. Label \(B\) also costs at least four. Finally, before
the first exchange of the entire route, all labels other than \(A\)
are stationary. To exchange with one of them, \(A\) must reach a
residual coordinate 0 or \(-2\). Its path from \((1,1)\), through
such a coordinate, to \((2,2)\) costs at least four.
These costs concern three distinct labels and can be summed, giving 12.
\end{proof}

\begin{proposition}[Terminal walking cost]
After the last restoration of \(C=T\), at least three further commands
are required to reach \(E\).
\end{proposition}
\begin{proof}
A restoration must be an exchange, since a walk does not change the
crystal set. Immediately before it, the player is at a missing target
\(t\in T\) and the exchanged crystal is at the sole off-target location
\(q\). The exchange leaves the player at \(q\), sharing a row or column
with \(t\). Hence at least one coordinate of \(q\) belongs to
\(\{4,\ldots,8\}\), and \(\norm{q-E}\ge3\).
Any subsequent exchange would put a crystal at the player's current
off-target location and remove one from a target, breaking \(C=T\).
Completion would then require a later restoration, a contradiction.
All remaining commands are walks, requiring at least \(\norm{q-E}\).
\end{proof}

\section{Ancillary files}
\label{app:files}
\begin{center}
\small
\begin{tabular}{p{.48\linewidth}p{.44\linewidth}}
\toprule
File & Role\\
\midrule
\texttt{crystal\_verify\_optimum.py} & Full reproduction driver\\
\texttt{crystal\_rules.py} & Independent exact transition model\\
\texttt{crystal\_lean\_cached\_bfs.cpp} & Final bounded-search implementation\\
\texttt{crystal\_encoded\_bfs.cpp} & Earlier checkpoint-compatible implementation\\
\texttt{crystal\_symmetry.cpp} & Mask representatives and stabilizers\\
\texttt{crystal\_oriented\_pdb.cpp} & Goal-oriented core distances\\
\texttt{crystal\_landmark\_pdb.cpp} & Northwest-visit distances\\
\texttt{crystal\_outer\_access.hpp} & Residual activation lower bound\\
\texttt{crystal\_occupancy\_codec.hpp} & Lossless occupied-set encoding\\
\texttt{crystal\_mapped.hpp} & Memory-mapped table storage\\
\texttt{crystal\_best\_verified.json} & 93-command witness\\
\texttt{crystal\_best\_verified\_steps.csv} & All intermediate witness positions\\
\texttt{crystal\_certificate\_bound93\_layers.txt} & Consolidated layer counts\\
\texttt{crystal\_bound93\_segment1.log}, etc. & Three original execution segments\\
\texttt{crystal\_optimality\_certificate.json} & Result, rule specification, and hashes\\
\bottomrule
\end{tabular}
\end{center}
Large regenerated pattern tables and binary checkpoints are omitted
from the submission archive. Their generators and table hashes are
included. The standalone manuscript embeds its diagrams and bibliography
and does not read the ancillary files during typesetting.

\end{document}